\documentclass{article} 
\usepackage[final]{colm2026_conference}
\usepackage{tabularx}
\usepackage{microtype}
\usepackage{hyperref}
\usepackage{url}
\usepackage{booktabs}
\usepackage{multirow}
\usepackage[table,dvipsnames]{xcolor}
\usepackage{graphicx}
\usepackage{subcaption}
\usepackage{wrapfig}
\usepackage{tikz}
\usetikzlibrary{shapes.geometric, arrows.meta, positioning, calc, fit, backgrounds, shadows}

\definecolor{avgrow}{RGB}{245, 245, 245}    

\usepackage{amsmath}
\usepackage{amssymb}
\usepackage{mathtools}
\usepackage{amsthm}
\usepackage{tcolorbox}

\usepackage[capitalize,noabbrev]{cleveref}

\usepackage{lineno}

\definecolor{darkblue}{rgb}{0, 0, 0.5}
\hypersetup{
  colorlinks=true,
  citecolor=darkblue,
  linkcolor=darkblue,
  urlcolor=darkblue,
  pdftitle={StateBridge: Training-free Hidden-state Alignment for Latent Communication in LLM Multi-Agent Systems},
  pdfauthor={Yanwen Peng, Delvin Ce Zhang, Xi Wang, Nikolaos Aletras}
}

\theoremstyle{plain}
\newtheorem{theorem}{Theorem}[section]
\newtheorem{proposition}[theorem]{Proposition}

\newtheorem{corollary}[theorem]{Corollary}
\theoremstyle{definition}

\theoremstyle{remark}

\title{StateBridge: Training-free Hidden-state Alignment for Latent Communication in LLM Multi-Agent Systems}

\author{
Yanwen Peng, Delvin Ce Zhang, Xi Wang, Nikolaos Aletras\\
School of Computer Science, University of Sheffield\\
Sheffield, United Kingdom\\
\texttt{\{ypeng86,delvin.ce.zhang,xi.wang,n.aletras\}@sheffield.ac.uk}
}

\begin{document}

\ifcolmsubmission
\linenumbers
\fi

\maketitle

\begin{abstract}
Large language model based multi-agent systems usually communicate in text, i.e., using discrete tokens. However, text introduces a discrete bottleneck. Converting the sender's continuous hidden states into discrete tokens discards information that token identities alone cannot capture. Recent work proposes latent communication as an alternative, where agents transmit hidden representations directly without converting them to text. However, existing latent methods either inject working memory layer by layer across the transformers, or require trained projectors that limit portability. We propose StateBridge, a training-free latent communication approach that aligns the sender's final-layer hidden states to the receiver's input space via a closed-form orthogonal transformation. Lightweight norm calibration and vocabulary anchoring ensure compatibility with the pretrained input distribution. The aligned states are prepended to the input of the receiver agent as a continuous prefix. We evaluate StateBridge on math reasoning, code generation, and question answering with four models from two families. StateBridge achieves the best or tied-best score on 22 out of 26 model–task pairs, consistently outperforming the strongest baseline.\footnote{Code is available at \url{https://github.com/YanwenPneg/StateBridge}.}
\end{abstract}

\section{Introduction}
\label{sec:intro}

Modern Multi-Agent (MA) systems consist of Large Language Model (LLM) agents that collaborate for tackling complex tasks requiring planning, critique, and verification \citep{wu2024autogen, li2023camel}. Yet their performance hinges on the communication channel between agents \citep{guo2024llmmultiagents, yan2025beyond}. Natural language remains the dominant medium, but text introduces a discrete bottleneck. The sender must transform its continuous internal state into tokens before transmission. This compression discards continuous information from the latent representation that discrete tokens alone do not capture \citep{zou2025latentcollaborationmultiagentsystems, du2025interlat}, impairing inter-agent coordination \citep{cemri2025multi}. The receiver maps this token sequence back to vectors through its own embedding layer, recovering only token identities rather than the sender's original hidden states. This results in information loss and motivates a fundamental question: \textit{can off-the-shelf LLM agents bypass natural language and effectively communicate through continuous representations?}

An intuitive strategy is to transmit hidden states directly. This is viable when all agents share the same pretrained LLM and thus the same hidden dimensionality \citep{zheng2025thought,zou2025latentcollaborationmultiagentsystems}. However, matching dimensionality alone is not sufficient. Pretrained LLMs are trained to read token embeddings at the input layer, whereas decoder hidden states occupy a different region of the representation space. Passing hidden states directly therefore produces latent vectors with the right dimensionality to receivers but lacks the geometric alignment required for the model to interpret them, causing semantic mismatch during  communication.

Prior work addresses this mismatch in two ways. Key-value (KV)-cache transfer methods \citep{fu2025cache, zou2025latentcollaborationmultiagentsystems} avoid the input-layer mismatch by injecting internal states across all transformer layers.  However, they transfer internal processing states rather than a compact representation of the complete message, which may potentially cause information loss. Learned latent-communication methods \citep{du2025interlat, zheng2025thought} bridge the gap through trained projectors, but the projectors tie the method to the specific model and tasks they were trained on and require re-training for different models. To the best of our knowledge, it has yet to be explored whether this mismatch can be resolved through alignment
alone, without training or architectural modification.

\begin{figure*}[t]
\centering
\includegraphics[width=0.7\textwidth]{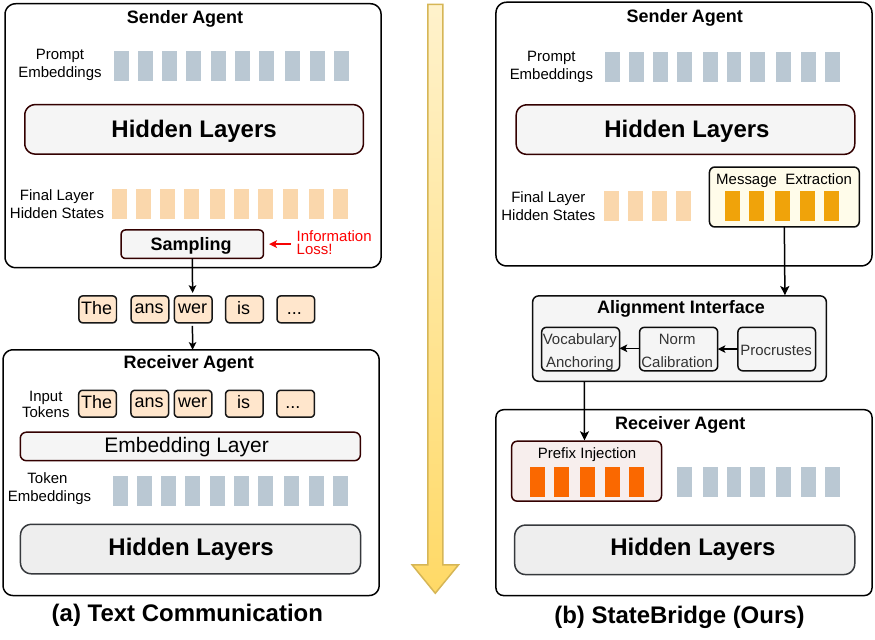}
\caption{\textbf{Overview of StateBridge.} Standard text communication (a) transforms the sender's message to discrete token identities via sampling, which causes information loss, whereas StateBridge (b) extracts message hidden states from the sender, aligns them with the input embedding space, and prepends the resulting prefix to the receiver's prompt.}
\label{fig:main}
\end{figure*}

We propose \textbf{StateBridge}, a training-free communication interface for MA systems (illustrated in Figure~\ref{fig:main}). We extract hidden states from the sender's generated message and align them with the receiver's input embedding space through a closed-form transformation. The aligned states are then injected as a continuous prefix for the receiver agent. Unlike text, it preserves a richer continuous representation from the sender. Compared to raw hidden-state transfer, it also makes such representation compatible with the receiver's input space. The central idea is that latent communication helps only when the transmitted states are both informative and compatible with the receiver's input space.

We evaluate StateBridge in a four-agent sequential pipeline on mathematical reasoning, code generation, and question answering benchmarks. Across four models spanning two families, StateBridge consistently outperforms text-based communication and KV-cache transfer methods. It achieves the best or tied-best score on 22 out of 26 model–task pairs without model updates and larger gains on challenging benchmarks. By operating only at the input embedding layer, StateBridge also applies more broadly across model families than methods that inject states across all transformer layers. A case study further confirms that the aligned prefix carries semantic information beyond the suffix tokens alone.

\paragraph{Contributions.} (1) We show that closed-form alignment alone resolves the representation mismatch between final-layer hidden states and input embeddings, enabling training-free and effective latent communication between off-the-shelf LLM agents. (2) We propose StateBridge, combining Procrustes alignment, norm calibration, and vocabulary anchoring into a closed-form alignment interface with no learnable parameters. (3) Across four models from two families, StateBridge outperforms both existing text-based and KV-cache transfer baselines. Ablations confirm that each component contributes, with geometry preservation providing the largest effect.

\section{Related work}
\label{sec:related_work}

\textbf{LLM MA systems.}
LLM MA systems extend classical multi-agent coordination \citep{park2023generative, yang2024llm} to modern language model settings, enabling autonomous agents to collaborate on reasoning, planning, and problem-solving \citep{tao2024magis, wang2025talk, zhao2025sirius, fang2025comprehensivesurveyselfevolvingai}. Early methods such as AutoGen \citep{wu2024autogen} and CAMEL \citep{li2023camel} coordinate multiple LLMs through explicit dialogue or role assignment. Subsequent work introduces structured communication protocols \citep{chen2025optima, yan2025beyond} and agent specialization \citep{mieczkowski2025predicting, fourney2024magentic}. These systems have been applied to math and science reasoning \citep{liang2024encouraging,yue2024clinicalagent}, open-domain question answering \citep{fourney2024magentic, wu2025talk}, and GUI interaction \citep{zhang2024large, ye2025mobile}. However, most existing methods communicate through text, which discards hidden state information \citep{du2025interlat}, introduces redundancy for linguistic coherence \citep{zhang2024cut}, and can impair inter-agent coordination \citep{cemri2025multi}. On the contrary, agents in our model communicate in the latent space, mitigating these problems.

\textbf{Latent communication in MA systems.}
Recent work explores latent-space communication as an alternative to text. Early work such as CIPHER \citep{pham2023cipher} transmits weighted average of vocabulary embeddings to avoid token sampling, but still projects onto the discrete vocabulary space rather than preserving hidden states. KV-cache methods such as Cache-to-Cache \citep{fu2025cache} and LatentMAS \citep{zou2025latentcollaborationmultiagentsystems}, which uses a training-free linear alignment to enable latent reasoning, transfer working memory across agents, but transmit internal processing state rather than communication content. Embedding-based methods such as Interlat \citep{du2025interlat} and ThoughtComm \citep{zheng2025thought} transmit hidden states directly, but require trained modules, limiting their applicability to the specific model and projector they were trained on. In contrast, our StateBridge transmits generated-message hidden states and enables effective message passing through training-free alignment to the receiver's input embedding space, offering effective content communication and wide generalizability.

\section{Method}
\label{sec:method}

\subsection{Problem statement}
\label{sec:problem}

A homogeneous MA system is defined by all agents sharing the same pretrained LLM $\mathcal{M}$ \citep{zhang2024chain, zou2025latentcollaborationmultiagentsystems}. $\mathcal{M}$ has vocabulary size $V$ and embedding dimension $d$, with the token embedding matrix defined as $\mathbf{W}_{\mathrm{emb}} \in \mathbb{R}^{V \times d}$. We abstract the communication channels within a MA system consisting of two primary entities: a sender $\mathcal{A}_s$ and a receiver $\mathcal{A}_r$. 
A message generated by the sender $\mathcal{A}_s$ is characterized by (1) the final-layer hidden states ($\mathbf{S}$) and (2) the embeddings of the corresponding decoded tokens in the shared $\mathbf{W}_{\mathrm{emb}}$, which we term the reference embeddings $\mathbf{R}$.

Sending only $\mathbf{R}$ reduces the message to discrete text, discarding continuous information that token identities alone do not capture (see detailed explanation in Appendix~\ref{app:theory_info}). Conversely, sending raw $\mathbf{S}$ preserves more information, but it is not aligned with the receiver's input embedding space, as the receiver is pretrained to consume token embeddings rather than last-layer hidden states.
Our goal is to produce an \emph{aligned prefix} $\bar{\mathbf{S}}$ that satisfies two requirements: (1) it preserves the pairwise geometric relationships among the sender's hidden states; and (2) it is compatible with the receiver's input embedding space, enabling its process without modification. StateBridge satisfies these two requirements as follows.

\subsection{Message extraction}

In autoregressive models, subsequent hidden states aggregate information from the preceding context. To leverage this, we retain only the final $K$ hidden states of the tokens of the generated message (e.g., $K=64$). These states are indexed from $1$ to $K$ and denoted by
\begin{equation}
\mathbf{S} = (\mathbf{s}_1, \ldots, \mathbf{s}_K)^\top \in \mathbb{R}^{K \times d},
\end{equation}
where $\mathbf{s}_i \in \mathbb{R}^d$ is the final-layer hidden state at position $i$. For models like Qwen3 \citep{yang2025qwen3} that generate intermediate reasoning (e.g., chain-of-thought delimited by $\langle$think$\rangle$ and $\langle$/think$\rangle$ tokens before the response), we discard hidden states from this section and retain the segment intended for the receiver. Let the resulting message tokens of the receiver be $\mathbf{y} = (y_1, \ldots, y_K)$. Looking up these tokens in the token embedding matrix gives
\begin{equation}
\mathbf{R} = \mathbf{W}_{\mathrm{emb}}[\mathbf{y}] \in \mathbb{R}^{K \times d}.
\end{equation}
The reference embeddings $\mathbf{R}$ serve as the optimization target for aligning the continuous message. The discrete tokens $\mathbf{y}$ are not transmitted between agents.

\begin{wrapfigure}[26]{r}{0.4\columnwidth}
\centering
\vspace{-12pt}
\includegraphics[width=0.38\columnwidth]{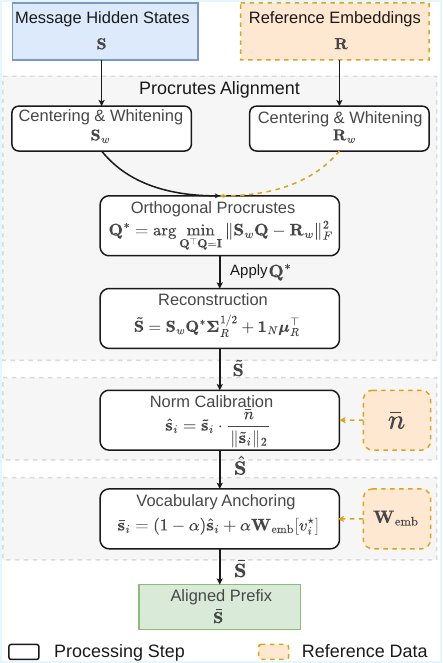}
\caption{The StateBridge alignment interface transforms message hidden states $\mathbf{S}$ into the aligned prefix $\bar{\mathbf{S}}$ for receiver using reference embeddings $\mathbf{R}$.}
\label{fig:workflow}
\vspace{-10pt}
\end{wrapfigure}

\subsection{Alignment interface}

The message hidden states $\mathbf{S}$ and the reference embeddings $\mathbf{R}$ describe the same generated message but occupy different regions of the representation space. The transformation must bridge this gap \citep{cao2025progressivedepthupscalingoptimal}. Since geometric proximity in hidden-state space encodes semantic similarity \citep{ethayarajh2019contextual}, the alignment transformation should therefore preserve pairwise geometric relationships rather than minimize pointwise reconstruction error. StateBridge achieves this through a three-step alignment interface illustrated in Figure~\ref{fig:workflow}: Procrustes alignment, norm calibration, and vocabulary anchoring.

\paragraph{Procrustes alignment.} 
We treat the message hidden states $\mathbf{S}$ and the reference embeddings $\mathbf{R}$ as two point clouds in $\mathbb{R}^d$. Both describe the same generated message, but in different coordinate systems. $\mathbf{S}$ carries the sender's richer internal representation, while $\mathbf{R}$ lies in the space the receiver can read. We align $\mathbf{S}$ to $\mathbf{R}$ using the orthogonal Procrustes method \citep{schonemann1966generalized}, which finds the rotation that best maps $\mathbf{S}$ onto $\mathbf{R}$ while preserving pairwise distances and angles. 

We first remove the global offset of each set by centering:
\begin{equation}
\mathbf{S}_c = \mathbf{S} - \mathbf{1}_K \boldsymbol{\mu}_S^\top,
\qquad
\mathbf{R}_c = \mathbf{R} - \mathbf{1}_K \boldsymbol{\mu}_R^\top,
\end{equation}
where $\mathbf{1}_K \in \mathbb{R}^K$ is the all-ones vector, $\boldsymbol{\mu}_S = \frac{1}{K}\sum_{i=1}^{K}\mathbf{s}_i$ is the row mean of $\mathbf{S} $, and $\boldsymbol{\mu}_R$ is defined analogously for $\mathbf{R}$. Centering restricts the alignment to the relative structure of the points rather than their absolute position difference in space.

We then whiten both sets:
\begin{equation}
\boldsymbol{\Sigma}_S = \frac{1}{K}\mathbf{S}_c^\top \mathbf{S}_c + \lambda \mathbf{I},
\qquad
\boldsymbol{\Sigma}_R = \frac{1}{K}\mathbf{R}_c^\top \mathbf{R}_c + \lambda \mathbf{I},
\end{equation}
\begin{equation}
\mathbf{S}_w = \mathbf{S}_c \boldsymbol{\Sigma}_S^{-1/2},
\qquad
\mathbf{R}_w = \mathbf{R}_c \boldsymbol{\Sigma}_R^{-1/2},
\end{equation}
where $\lambda > 0$ is a small regularization constant, and $\mathbf{I}$ is the identity matrix. Whitening rescales the principal directions of each set so that a few high-variance directions do not dominate the alignment. Centering and whitening together remove the global offset and equalize the principal directions. The remaining mismatch between the two point clouds is then a rotation.

We then solve an orthogonal Procrustes problem:
\begin{equation}
\mathbf{Q}^*
=
\arg\min_{\mathbf{Q}}
\left\|
\mathbf{S}_w \mathbf{Q} - \mathbf{R}_w
\right\|_F^2\quad\text{s.t.}\quad \mathbf{Q}^\top \mathbf{Q} = \mathbf{I},
\end{equation}
where $\|\cdot\|_F$ denotes the Frobenius norm. This objective seeks the rotation that best aligns the whitened sender states to the whitened reference embeddings without introducing arbitrary shearing. The solution is closed-form. Let $\mathbf{S}_w^\top \mathbf{R}_w = \mathbf{U}\mathbf{D}\mathbf{V}^\top$ be the singular value decomposition (SVD) \citep{golub2013matrix} of the cross-correlation matrix. The optimal rotation is then $\mathbf{Q}^* = \mathbf{U}\mathbf{V}^\top$. Since $\mathbf{Q}^*$ is orthogonal, it preserves pairwise distances and angles in the whitened space. We provide a formal proof in Appendix~\ref{app:theory_procrustes}.

Finally, we restore the overall scale pattern and location of the reference embeddings:
\begin{equation}
\tilde{\mathbf{S}}
=
\mathbf{S}_w \mathbf{Q}^* \boldsymbol{\Sigma}_R^{1/2}
+
\mathbf{1}_K \boldsymbol{\mu}_R^\top.
\end{equation}
This Procrustes alignment ensures that the resulting matrix $\tilde{\mathbf{S}}$ remains a continuous representation, derived from the sender's hidden states, while being closer to the input geometry required by the receiver for effective communication.

\paragraph{Norm calibration and vocabulary anchoring.}

The Procrustes alignment step matches the orientation of the message to the input embedding space, but two practical issues remain: (1) Final-layer hidden states have much larger norms than token embeddings. On Qwen3-4B \citep{yang2025qwen3}, the average hidden-state norm is approximately 140$\times$ that of the input embeddings. This gap arises from the accumulation of residual updates across layers and the next-token prediction objective. The norm gap persists after alignment and changes how strongly the aligned vectors interact with the receiver's attention layers. (2) Aligned vectors $\tilde{\mathbf{S}}$ could still lie far from regions occupied by actual vocabulary embeddings. We address these issues by applying a two-step adjustment to each aligned vector.

First, we calibrate each aligned vector to the typical norm of the vocabulary embeddings. Let $\bar{n} = \frac{1}{V}\sum_{v=1}^{V} \|\mathbf{W}_{\mathrm{emb}}[v]\|_2$ be the average $\ell_2$ norm of the vocabulary embeddings. For each row $\tilde{\mathbf{s}}_i$ of $\tilde{\mathbf{S}}$, we compute
\begin{equation}
\hat{\mathbf{s}}_i
=
\tilde{\mathbf{s}}_i
\cdot
\frac{\bar{n}}{\|\tilde{\mathbf{s}}_i\|_2}.
\end{equation}
This step places the prefix on a comparable norm scale to the ordinary input token embeddings. Without it, the larger prefix norms would dominate the dot-product attention scores, causing the receiver to attend to the prefix regardless of semantic relevance.

Second, we move each calibrated vector slightly toward its nearest vocabulary embedding under cosine similarity:
\begin{equation}
v_i^*
=
\arg\max_{v \in \{1,\ldots,V\}}
\frac{
\hat{\mathbf{s}}_i^\top \mathbf{W}_{\mathrm{emb}}[v]
}{
\|\hat{\mathbf{s}}_i\|_2 \, \|\mathbf{W}_{\mathrm{emb}}[v]\|_2
},
\qquad
\bar{\mathbf{s}}_i
=
(1-\alpha)\hat{\mathbf{s}}_i + \alpha \mathbf{W}_{\mathrm{emb}}[v_i^*],
\end{equation}
where $\alpha \in [0,1]$. This vocabulary anchoring step does not snap the message hidden states to discrete tokens. Instead, it keeps the representation continuous while moving it closer to regions of the input embedding space encountered during pretraining. While Procrustes alignment preserves the geometric structure of the message hidden states, these two steps ensure that the prefix is also compatible with the receiver's input embedding space. Stacking all rows $\bar{\mathbf{s}}_i$ gives the final aligned prefix $\bar{\mathbf{S}} \in \mathbb{R}^{K \times d}$.

\subsection{Prefix injection}

Let $\mathbf{P} \in \mathbb{R}^{N \times d}$ denote the receiver's prompt embeddings, where $N$ is the number of prompt tokens. StateBridge prepends the aligned prefix to the prompt, forming the combined input $\mathbf{X} = [\bar{\mathbf{S}}; \mathbf{P}] \in \mathbb{R}^{(K+N) \times d}$. StateBridge injects the aligned prefix directly at the embedding layer, bypassing tokenization. We assign position indices sequentially over the concatenated input, so the prefix occupies positions $0$ through $K{-}1$, and the prompt follows from position $K$ onward. The receiver's attention and position-encoding mechanisms therefore treat the prefix identically to ordinary token embeddings, requiring no special handling. The receiver processes $\mathbf{X}$ with the standard  language modeling layers, without any architectural modification. StateBridge therefore changes only the communication interface between agents. The sender transmits an aligned continuous prefix instead of text, while the underlying LLM remains unchanged.

\subsection{Computational cost}

\textbf{Time complexity.} The alignment is dominated by two operations: the whitening eigendecomposition, which costs $O(d^3)$, and the vocabulary anchoring search, which costs $O(KVd)$. Both computed once per batch. The remaining operations (centering, rotation, norm calibration) are $O(Kd^2)$ or lower. For comparison, let $T$ denote the number of generated tokens and $L$ the number of transformer layers. A single autoregressive pass costs $O(TLd^2)$. For typical configurations ($T \geq 256$, $L = 32$), alignment is cheaper than one generation pass.

\textbf{Space complexity.} Naively storing every layer's hidden states during generation costs $O(TLd)$. We instead register a lightweight callback (forward hook) on the final transformer layer. It records only that layer's output at each decoding step and discards the rest, reducing the extraction cost to $O(Td)$. Only the last $K$ states are retained, giving $O(Kd)$. The alignment adds $O(d^2)$ for covariance and SVD factors. Both are negligible relative to the model parameters. By contrast, KV-cache transfer methods store $O(TLd)$ per agent \citep{fu2025cache, zou2025latentcollaborationmultiagentsystems}, incurring substantially larger memory overhead.

\section{Experimental setup}
\label{sec:exp_setup}

We follow \citet{zou2025latentcollaborationmultiagentsystems} and use a standard four-agent pipeline across reasoning, question answering, and code generation tasks.

\textbf{Tasks and datasets.}
We use eight benchmarks from three task categories: mathematical reasoning (GSM8K \citep{gsm8k}, AIME24 \citep{aime24}, and AIME25 \citep{aime25}), question answering (GPQA-Diamond \citep{gpqa}, ARC-Challenge \citep{arc-challenge}, and MedQA \citep{medqa}), and code generation (MBPP+ and HumanEval+ \citep{codeplus}). This selection covers tasks that differ in reasoning depth, knowledge demand, and output precision. Detailed dataset descriptions are provided in Appendix~\ref{app:datasets}.

\textbf{Models.}
We test two model families: Qwen3 (4B, 8B, and 32B) \citep{yang2025qwen3} and OLMo3-7B-Think \citep{olmo2025olmo3}. Within each run, all agents share the same model weights. We evaluate all models on five core benchmarks. Following \citet{zou2025latentcollaborationmultiagentsystems}, we additionally report AIME24, AIME25, and GPQA-Diamond for the stronger Qwen3-8B and Qwen3-32B models.

\textbf{Baselines.}
We compare StateBridge against three baselines: (i) \textit{Single}, standard single-agent generation; (ii) \textit{TextMAS}, text-based sequential communication \citep{zhang2024chain}; and (iii) \textit{LatentMAS} \citep{zou2025latentcollaborationmultiagentsystems}, which injects the sender's KV cache into the receiver. All methods share the same four-agents (Planner, Critic, Refiner, Judger), identical hyperparameters, and the same evaluation protocol. The only difference across prompts is the communication modality description. Appendix~\ref{app:agent_pipeline} describes the agent pipeline and Appendix~\ref{app:prompts} lists the full prompt templates.

\textbf{Implementation details.}
We set $K = 64$, $\alpha = 0.3$, and $\lambda = 10^{-3}$. Following \citet{zou2025latentcollaborationmultiagentsystems}, all agents use temperature 0.6 and top-$p$ 0.95. Maximum output lengths range from 2,048 to 20,000 tokens depending on the task (Appendix~\ref{app:impl_details}). All experiments were run on 2 NVIDIA A100-80G GPUs. The evaluation protocol for each task category is described in Appendix~\ref{app:eval}.

\section{Results}
\label{sec:results}

\begin{table*}[t!]
\centering
\scriptsize
\renewcommand{\arraystretch}{1.05}

\newcolumntype{C}{>{\centering\arraybackslash}X}
\newcolumntype{D}{>{\centering\arraybackslash}p{3.2em}}

\begin{tabularx}{\textwidth}{l CCCCD CCCCD}
\toprule
& \multicolumn{5}{c}{\textbf{Qwen3-4B}} & \multicolumn{5}{c}{\textbf{OLMo3-7B-Think}} \\
\cmidrule(lr){2-6} \cmidrule(lr){7-11}
\textbf{Task} & Single & Text & Latent & \textbf{SB} & $\Delta$ & Single & Text & Latent & \textbf{SB} & $\Delta$ \\
\midrule
ARC-C & 89.2 & 90.0 & 92.3 & \textbf{93.7} & \textcolor{ForestGreen}{$\uparrow$1.4} & 84.6 & 82.2 & 78.3 & \textbf{89.6} & \textcolor{ForestGreen}{$\uparrow$5.0} \\
MedQA & 47.7 & 65.3 & 66.3 & \textbf{70.3} & \textcolor{ForestGreen}{$\uparrow$4.0} & 50.7 & 54.0 & 44.3 & \textbf{59.0} & \textcolor{ForestGreen}{$\uparrow$5.0} \\
GSM8K & 82.4 & \textbf{89.8} & 88.2 & \textbf{89.8} & 0.0 & 85.7 & 85.7 & 67.5 & \textbf{86.5} & \textcolor{ForestGreen}{$\uparrow$0.8} \\
MBPP+ & 63.5 & 69.8 & 73.5 & \textbf{75.9} & \textcolor{ForestGreen}{$\uparrow$2.4} & 66.7 & 66.9 & 42.3 & \textbf{69.1} & \textcolor{ForestGreen}{$\uparrow$2.2} \\
HumanEval+ & 75.0 & 79.7 & 79.9 & \textbf{82.3} & \textcolor{ForestGreen}{$\uparrow$2.4} & 71.3 & \textbf{80.5} & 43.3 & 79.3 & \textcolor{BrickRed}{$\downarrow$1.2} \\
\midrule
\rowcolor{avgrow} \textbf{Avg.} & 71.6 & 78.9 & 80.0 & \textbf{82.4} & \textcolor{ForestGreen}{\textbf{$\uparrow$2.4}} & 71.8 & 73.9 & 55.1 & \textbf{76.7} & \textcolor{ForestGreen}{\textbf{$\uparrow$2.8}} \\
\midrule
\midrule
& \multicolumn{5}{c}{\textbf{Qwen3-8B}} & \multicolumn{5}{c}{\textbf{Qwen3-32B}} \\
\cmidrule(lr){2-6} \cmidrule(lr){7-11}
\textbf{Task} & Single & Text & Latent & \textbf{SB} & $\Delta$ & Single & Text & Latent & \textbf{SB} & $\Delta$ \\
\midrule
ARC-C & 91.0 & 94.6 & 94.4 & \textbf{95.1} & \textcolor{ForestGreen}{$\uparrow$0.5} & 95.5 & 93.9 & \textbf{95.7} & 94.5 & \textcolor{BrickRed}{$\downarrow$1.2} \\
MedQA & 53.0 & 75.0 & 75.3 & \textbf{78.3} & \textcolor{ForestGreen}{$\uparrow$3.0} & 86.0 & 77.0 & 84.3 & \textbf{87.3} & \textcolor{ForestGreen}{$\uparrow$1.3} \\
GSM8K & 81.1 & 92.3 & \textbf{93.8} & 91.2 & \textcolor{BrickRed}{$\downarrow$2.6} & 91.1 & 92.4 & \textbf{92.7} & 90.2 & \textcolor{BrickRed}{$\downarrow$2.5} \\
MBPP+ & 64.8 & 69.5 & 74.6 & \textbf{76.9} & \textcolor{ForestGreen}{$\uparrow$2.3} & 74.3 & 75.1 & 74.1 & \textbf{76.2} & \textcolor{ForestGreen}{$\uparrow$1.1} \\
HumanEval+ & 74.4 & 80.5 & 80.5 & \textbf{83.6} & \textcolor{ForestGreen}{$\uparrow$3.1} & 78.7 & 84.8 & 84.2 & \textbf{85.4} & \textcolor{ForestGreen}{$\uparrow$0.6} \\
\cmidrule{1-6} \cmidrule{7-11}
AIME24 & 50.0 & 53.3 & 56.7 & \textbf{63.3} & \textcolor{ForestGreen}{$\uparrow$6.6} & 70.0 & 73.3 & 66.7 & \textbf{76.7} & \textcolor{ForestGreen}{$\uparrow$3.4} \\
AIME25 & 46.7 & \textbf{53.3} & \textbf{53.3} & \textbf{53.3} & 0.0 & 66.7 & 70.0 & 56.7 & \textbf{73.3} & \textcolor{ForestGreen}{$\uparrow$3.3} \\
GPQA & 39.9 & 43.4 & 45.5 & \textbf{52.5} & \textcolor{ForestGreen}{$\uparrow$7.0} & 52.5 & 58.3 & 57.1 & \textbf{64.1} & \textcolor{ForestGreen}{$\uparrow$5.8} \\
\midrule
\rowcolor{avgrow} \textbf{Avg.} & 62.6 & 70.2 & 71.8 & \textbf{74.3} & \textcolor{ForestGreen}{\textbf{$\uparrow$2.5}} & 76.9 & 78.1 & 76.4 & \textbf{81.0} & \textcolor{ForestGreen}{\textbf{$\uparrow$2.9}} \\
\bottomrule
\end{tabularx}
\caption{Main results across model families. Accuracy (\%) on QA and math tasks, pass@1 (\%) on code tasks. \textbf{Bold}: best or tied-best per task. $\Delta$: StateBridge vs.\ the best baseline (\textcolor{ForestGreen}{green} = gain, \textcolor{BrickRed}{red} = loss). Text = TextMAS, Latent = LatentMAS, SB = StateBridge (Ours).}
\label{tab:main}
\end{table*}


Table~\ref{tab:main} summarizes the results across MA communication methods, models, and tasks. We report accuracy for question answering and mathematical reasoning and pass@1 \citep{chen2021evaluating} for code generation. StateBridge achieves the best average score in all four model settings, improving over the best baseline by 2.4 to 2.9 points. It also wins or ties on 22 out of 26 model--task pairs. These gains hold across both model families.

Switching from Qwen3 to OLMo3 tests whether each latent method is sensitive to the LLM architecture. On OLMo3-7B-Think, LatentMAS averages only 55.1\%, far below the 73.9\% of TextMAS. StateBridge reaches 76.7\%. Because the agent pipeline, prompts, and decoding settings are otherwise unchanged, this gap points to the communication channel. KV-cache transfer injects states across each transformer layers. If the layer structure differs between LLMs, compatibility breaks. StateBridge operates only at the input embedding layer and the use of such a uniform interface ensures a greater applicability across model families.


StateBridge helps most on the more challenging benchmarks, including MedQA, GPQA, AIME24/25, and the code generation tasks. StateBridge's message hidden states encode distributional information beyond the sampled tokens, including confidence signals and traces of alternative reasoning paths. Text-based communication collapses this to a single token sequence, leading to information loss and incomplete message communication.


Not all tasks show gains. On GSM8K for Qwen3 models, StateBridge trails the best baseline, likely because continuous prefixes affect output formatting under exact-match evaluation \citep{li2021prefix,tam2024speak}.

\section{Analysis}
\label{sec:analysis}
\subsection{Ablation study}

Table~\ref{tab:ablation} isolates the contribution of each StateBridge component.

\textbf{Orthogonal vs.\ linear alignment.} A natural alternative to Procrustes is to fit a linear map via ridge regression, minimizing $|\mathbf{SW} - \mathbf{R}|^2 + \lambda|\mathbf{W}|^2$. This map is not constrained to preserve distances or angles, so it achieves low pointwise reconstruction error. However, it distorts the pairwise geometry among the sender states. 
Replacing Procrustes with ridge regression results to average performance drops from 82.4\% to 74.9\%, with especially larger drops on MBPP+ and HumanEval+. Since pairwise geometry encodes semantic similarity, distorting it discards the richer information that latent communication carries over text. The effect is strongest on code generation, which demands structurally precise output. Orthogonal alignment preserves these semantic similarity and maintains the precision of the receiver's generation. Appendices~\ref{app:theory_ridge} and~\ref{app:theory_procrustes} provide a geometric explanation for this gap.

\textbf{Compatibility steps.} The two compatibility steps also matter. Removing norm calibration lowers the average to 79.5\%, and removing vocabulary anchoring lowers it to 80.2\%. This suggests that alignment alone is not enough. The prefix also needs the right norm scale and must stay close to embedding regions seen during pretraining. Replacing the aligned prefix with random noise drops the average to 48.8\%. This rules out the possibility that gains come from simply prepending extra continuous vectors.

\begin{table*}[t!]
\centering
\scriptsize
\setlength{\tabcolsep}{4pt}
\renewcommand{\arraystretch}{1.0}

\newcolumntype{C}{>{\centering\arraybackslash}X}
\begin{tabularx}{\textwidth}{l C C C C C C l}
\toprule
\textbf{Variant} & ARC-C & MedQA & GSM8K & MBPP+ & HE+ & Avg. & $\Delta$ \\
\midrule
\rowcolor{avgrow} Full StateBridge & \textbf{93.7} & \textbf{70.3} & \textbf{89.8} & \textbf{75.9} & \textbf{82.3} & \textbf{82.4} & \multicolumn{1}{c}{--} \\
Ridge Regr. & 93.0 & 68.0 & 88.6 & 61.5 & 63.6 & 74.9 & \textcolor{BrickRed}{$\downarrow$7.5} \\
\midrule
Random Noise & 32.2 & 30.7 & 84.5 & 48.2 & 48.2 & 48.8 & \textcolor{BrickRed}{$\downarrow$33.6} \\
w/o Norm Calib. & 90.9 & 68.0 & 87.1 & 71.7 & 79.9 & 79.5 & \textcolor{BrickRed}{$\downarrow$2.9} \\
w/o Vocab.\ Anchor. & 91.7 & 66.7 & 88.9 & 73.0 & 80.5 & 80.2 & \textcolor{BrickRed}{$\downarrow$2.2} \\
\bottomrule
\end{tabularx}
\caption{Ablation on Qwen3-4B. Each row removes or replaces one component. HE+: HumanEval+.}
\label{tab:ablation}
\end{table*}

\begin{table*}[t]
\centering
\scriptsize
\setlength{\tabcolsep}{3pt}
\renewcommand{\arraystretch}{1.0}

\newcolumntype{C}{>{\centering\arraybackslash}X}

\begin{tabularx}{\textwidth}{l C C C C c C C C C C C}
\toprule
& \multicolumn{4}{c}{\textbf{Prefix length} $K$} && \multicolumn{6}{c}{\textbf{Anchoring coefficient} $\alpha$} \\
\cmidrule(lr){2-5} \cmidrule(lr){7-12}
\textbf{Task} & 16 & 32 & \cellcolor{avgrow}\textbf{64} & 128 && 0.0 & 0.1 & 0.2 & \cellcolor{avgrow}\textbf{0.3} & 0.4 & 0.6 \\
\midrule
ARC-C & 93.2 & 93.3 & \cellcolor{avgrow}\textbf{93.7} & 91.7 && 91.7 & 93.5 & 93.2 & \cellcolor{avgrow}\textbf{93.7} & 92.5 & 93.3 \\
MedQA & 68.0 & 68.7 & \cellcolor{avgrow}\textbf{70.3} & 68.7 && 66.7 & 64.0 & \textbf{71.3} & \cellcolor{avgrow}70.3 & 69.7 & 66.3 \\
GSM8K & 87.6 & 87.9 & \cellcolor{avgrow}\textbf{89.8} & 88.6 && 88.9 & 89.1 & 88.7 & \cellcolor{avgrow}\textbf{89.8} & 88.0 & 89.0 \\
MBPP+ & 74.6 & \textbf{75.9} & \cellcolor{avgrow}\textbf{75.9} & 74.9 && 73.0 & 73.3 & 69.3 & \cellcolor{avgrow}\textbf{75.9} & 72.2 & 72.0 \\
HumanEval+ & \textbf{82.3} & 81.7 & \cellcolor{avgrow}\textbf{82.3} & 79.3 && 80.5 & 78.1 & 81.7 & \cellcolor{avgrow}82.3 & \textbf{83.5} & 82.9 \\
\bottomrule
\end{tabularx}
\caption{Sensitivity to prefix length $K$ and anchoring coefficient $\alpha$ on Qwen3-4B. Moderate values work best. Default values ($K{=}64$, $\alpha{=}0.3$) are shaded.}
\label{tab:hyperparam}
\end{table*}

\subsection{Hyperparameter sensitivity}
\label{sec:hyperparam}
For all experiments, we use the best hyperparameters ($K{=}64$, $\alpha{=}0.3$) on MedQA obtained with Qwen3-4B and apply them without modification to all other datasets and models. Table~\ref{tab:hyperparam} reports the sensitivity to prefix length $K$ and anchoring coefficient $\alpha$. Increasing $K$ from 16 to 64 usually improves performance, suggesting that very short messages loss useful information. However, $K{=}128$ hurts on most tasks. We attribute this to the alignment procedure. StateBridge fits a single global rotation to all $K$ states. When \(K\) is small, retained states come from the message end and tend to have similar norms and geometry. Larger \(K\) includes earlier states that may differ in both. The global rotation then has to compromise across a more heterogeneous point set, potentially lowering alignment quality. 

The anchoring coefficient $\alpha$ controls a trade-off between informativeness and compatibility. A smaller $\alpha$ preserves more of the aligned hidden state but keeps the prefix further from the receiver's input embedding space. A larger $\alpha$ moves the prefix closer to actual reference embeddings but discards more continuous information.
$\alpha{=}0.3$ strikes a consistent balance across all benchmarks, so we adopt it as a fixed default.

\subsection{Alignment visualization}

Figure~\ref{fig:alignment_viz} shows the gap between message hidden states and reference embeddings, and how StateBridge closes it. States are collected from 300 MedQA queries. Before alignment, the message hidden states occupy a different region from the reference embeddings. After alignment, the prefix moves much closer to the input embedding space. The same pattern appears for both Qwen3-4B and Qwen3-8B, suggesting that the effect is stable across model scales.

\begin{figure*}[t!]
\centering
\includegraphics[width=0.85\textwidth]{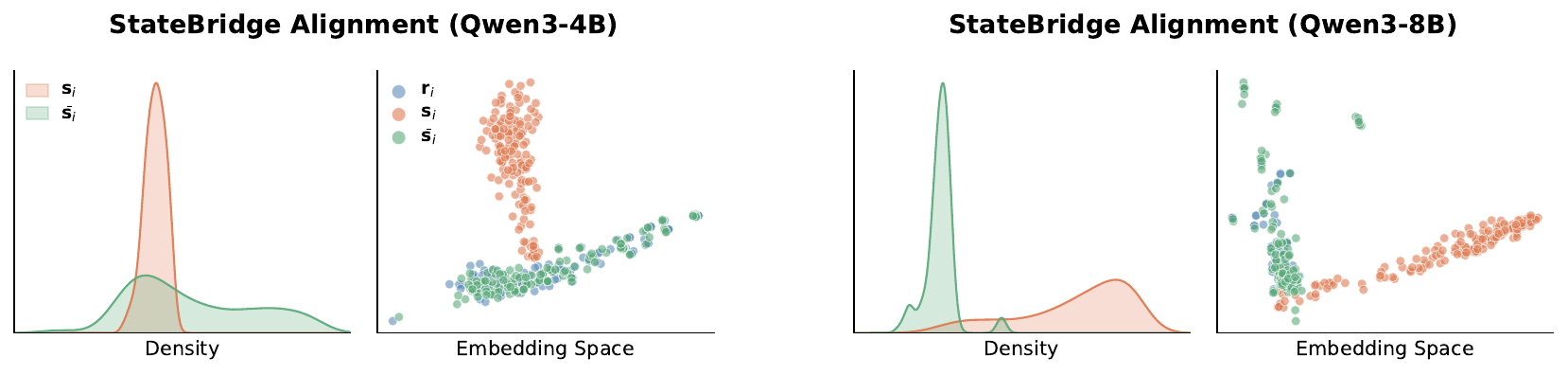}
\caption{Principal component analysis (PCA) density (left) and PCA scatter of the embedding space (right) for Qwen3-4B and Qwen3-8B. Orange: message hidden states $\mathbf{s}_i$; blue: reference embeddings $\mathbf{r}_i$; green: aligned prefix $\bar{\mathbf{s}}_i$.}
\label{fig:alignment_viz}
\end{figure*}

\subsection{Case study: plan recovery from aligned prefixes}

\begin{table*}[!t]
\centering
\scriptsize
\setlength{\tabcolsep}{4pt}
\renewcommand{\arraystretch}{1.4}
\begin{tabular*}{\textwidth}{@{\extracolsep{\fill}}p{0.09\textwidth}@{\quad} p{0.28\textwidth} p{0.28\textwidth} p{0.28\textwidth}@{}}
\toprule
\textbf{Original Plan} \newline (185\newline tokens)
& \multicolumn{3}{p{0.85\textwidth}}{1.\ Analyze clinical features (fatigue, dysphagia, weight loss, esophageal web, flat nails). 2.\ Evaluate lab results (low hemoglobin, normal WBC/ESR). 3.\ Correlate: Plummer-Vinson syndrome linked to anemia, web, dysphagia; other diagnoses less likely. 4.\ Confirm diagnosis. ...} \\
\midrule
& \multicolumn{1}{c}{$\mathbf{K{=}16}$} & \multicolumn{1}{c}{$\mathbf{K{=}64}$} & \multicolumn{1}{c}{$\mathbf{K{=}128}$} \\
\midrule
\textbf{Suffix\newline Tokens}
& \textit{\ldots al web, and systemic symptoms aligns most strongly with Plummer-Vinson.}
& \textit{\ldots Plummer-Vinson is linked to iron deficiency, esophageal webs, and oral candidiasis, aligning with the patient's clinical picture.}
& \textit{\ldots Low hemoglobin with normal WBC and ESR supports chronic anemia.\ \ldots\ Iron deficiency in Plummer-Vinson is due to strictures or blood loss, aligning with the patient's symptoms.} \\
\midrule
\textbf{Critic} \newline \textbf{Recovery}
& \ldots\ the most likely diagnosis based on the clinical features (anemia, dysphagia, \textbf{koilonychia}, upper esophageal web) \ldots\ The differential diagnoses (e.g., esophageal cancer, achalasia) are less consistent with the presented findings.
& The plan \ldots\ noting \textbf{its association with iron deficiency, esophageal webs, and oral candidiasis}, aligning with \textbf{(fatigue, weight loss, flat nails, and esophageal web)}. It also \textbf{differentiates this from esophageal cancer or achalasia} \ldots
& \ldots\ features like the web or specific malignancy indicators. Confirm associations: Iron deficiency anemia in Plummer-Vinson is often due to strictures or blood loss, aligning with the patient's symptoms. \\
\bottomrule
\end{tabular*}
\caption{Case study on a MedQA example (Qwen3-4B). Suffix Tokens shows the last $K$ decoded tokens; Critic Recovery shows the critic's restatement from the aligned prefix alone. Bold highlights information absent from the visible suffix.}
\label{tab:case_study}
\end{table*}

Table~\ref{tab:case_study} shows a controlled reconstruction example. We ask the critic to restate the planner's original plan using only the aligned prefix received through StateBridge, without access to the original text. We then vary the prefix length $K$. At $K{=}16$, the visible suffix tokens form only a short fragment. Yet the critic recovers the diagnosis, the key clinical features, and the exclusion of alternative diagnoses. The restatement even includes \textit{koilonychia}, a more precise medical term for ``flat nails.'' This term does not appear in the visible suffix tokens. At $K{=}64$, the recovered plan becomes more detailed, and the critic additionally identifies differential diagnoses. At $K{=}128$, the critic copies only a truncated, incomplete plan. The aligned prefix no longer carries continuous information beyond the suffix tokens, consistent with the reduced alignment quality discussed in \cref{sec:hyperparam}.

\section{Conclusion}
\label{sec:conclusion}

We introduced StateBridge, a training-free communication interface that aligns sender hidden states to the receiver's input embedding space. It requires no training and no architectural modification. Across four model settings and eight benchmarks, StateBridge consistently outperforms both text-based and KV-cache baselines. The gains are largest on harder tasks, where a text handoff is most likely to discard useful information. Together with the ablations and case study, these results point to a key insight. The challenge in latent communication is not only sending richer states. It is making those states readable to the next agent. More broadly, the results show that communication interface design matters in homogeneous MA systems. Our current setting assumes a shared model and a fixed prefix length. Future work could extend the method to heterogeneous models, choose the prefix length adaptively, and combine training-free alignment with learned communication modules.

\bibliographystyle{colm2026_conference}
\bibliography{references}

\appendix
\section{Theoretical analysis}
\label{app:theory}

In this section, we analyze the geometric properties of two alignment strategies used in experiments: orthogonal Procrustes and ridge regression. We first formalize the information bottleneck in text communication, then contrast the two approaches at the representation level.

\subsection{Information loss in text communication}
\label{app:theory_info}

Standard text communication maps the sender's hidden states $\mathbf{S} \in \mathbb{R}^{K \times d}$ to a discrete token sequence $\mathbf{y} \in \{1, \ldots, V\}^K$ via sampling. This mapping is many-to-one: distinct hidden-state matrices that produce the same token sequence become indistinguishable to the receiver. We treat $\mathbf{S}$ and $\mathbf{y}$ as random variables induced by the data distribution and decoding randomness. Since $\mathbf{W}_{\mathrm{emb}}[\mathbf{y}]$ is a deterministic function of $\mathbf{y}$, the data processing inequality bounds the mutual information that the inter-agent channel carries:
\begin{equation}
I(\mathbf{S};\, \mathbf{W}_{\mathrm{emb}}[\mathbf{y}]) \le I(\mathbf{S};\, \mathbf{y}) \le H(\mathbf{y}) \le K \log_2 V,
\end{equation}
where the first inequality becomes an equality under injective embedding lookup. When communication is restricted to the retained $K$ sampled tokens, the channel carries at most $K \log_2 V$ bits; for $V \approx 1.5 \times 10^5$ and $K = 64$, this gives roughly $1101$ bits, regardless of how much information the $K \times d$ continuous hidden states contain. Continuous transmission is not subject to the same combinatorial bound, although its effective capacity depends on numerical precision and the receiver's ability to interpret the representation.

\subsection{Confinement to the token-embedding span under ridge alignment}
\label{app:theory_ridge}

We now show that replacing Procrustes with an unconstrained linear map confines the transmitted message to the span of the discrete token embeddings, regardless of the regularization strength. Consider ridge regression alignment, which solves
\begin{equation}
\mathbf{W}_{\mathrm{ridge}} = \arg\min_{\mathbf{W} \in \mathbb{R}^{d \times d}} \|\mathbf{S}\mathbf{W} - \mathbf{R}\|_F^2 + \gamma \|\mathbf{W}\|_F^2,
\end{equation}
with closed-form solution
\begin{equation}
\mathbf{W}_{\mathrm{ridge}} = (\mathbf{S}^\top \mathbf{S} + \gamma \mathbf{I})^{-1} \mathbf{S}^\top \mathbf{R}.
\end{equation}

\begin{proposition}[Span Confinement]
\label{prop:collapse}
The ridge-aligned prefix $\mathbf{S}_{\mathrm{ridge}} = \mathbf{S} \mathbf{W}_{\mathrm{ridge}}$ satisfies
\begin{equation}
\mathbf{S}_{\mathrm{ridge}} = \mathbf{H}_\gamma \mathbf{R},
\end{equation}
where $\mathbf{H}_\gamma = (\mathbf{S}\mathbf{S}^\top + \gamma \mathbf{I})^{-1} \mathbf{S}\mathbf{S}^\top \in \mathbb{R}^{K \times K}$. Consequently, every row of\, $\mathbf{S}_{\mathrm{ridge}}$ is a linear combination of the rows of\, $\mathbf{R}$, and the aligned prefix is confined to $\mathrm{span}(\mathbf{r}_1, \ldots, \mathbf{r}_K)$.
\end{proposition}
\begin{proof}
Applying the push-through identity $\mathbf{A}(\mathbf{A}^\top \mathbf{A} + \gamma \mathbf{I})^{-1} = (\mathbf{A}\mathbf{A}^\top + \gamma \mathbf{I})^{-1}\mathbf{A}$ with $\mathbf{A} = \mathbf{S}$:
\begin{equation}
\mathbf{S}_{\mathrm{ridge}}
= \mathbf{S}(\mathbf{S}^\top \mathbf{S} + \gamma \mathbf{I})^{-1} \mathbf{S}^\top \mathbf{R}
= (\mathbf{S}\mathbf{S}^\top + \gamma \mathbf{I})^{-1} \mathbf{S}\mathbf{S}^\top \mathbf{R}
= \mathbf{H}_\gamma \mathbf{R}.
\end{equation}
Since $\mathbf{H}_\gamma \in \mathbb{R}^{K \times K}$, row $i$ of $\mathbf{S}_{\mathrm{ridge}}$ equals $\sum_{j=1}^K (\mathbf{H}_\gamma)_{ij}\, \mathbf{r}_j \in \mathrm{span}(\mathbf{r}_1, \ldots, \mathbf{r}_K)$.
\end{proof}

\begin{corollary}[Limiting Behavior]
\label{cor:degeneration}
When $\mathbf{S}$ has full row rank (which holds generically for $K \le d$ and rows in general position), $\mathbf{H}_\gamma \to \mathbf{I}_K$ as $\gamma \to 0$, so $\mathbf{S}_{\mathrm{ridge}} \to \mathbf{R}$. In the opposite limit, $\mathbf{H}_\gamma \to \mathbf{0}$ as $\gamma \to \infty$, so the prefix vanishes entirely.
\end{corollary}

Because $\mathbf{R}$ contains the embeddings of the $K$ sampled text tokens, its row space is at most $K$-dimensional and is determined entirely by the discrete token identities. Proposition~\ref{prop:collapse} shows that the ridge-aligned prefix is confined to this span for all $\gamma$, while Corollary~\ref{cor:degeneration} establishes the two limiting regimes: text reconstruction ($\gamma \to 0$) and silence ($\gamma \to \infty$). For intermediate $\gamma$, the $\mathbf{S}$-dependent coefficients in $\mathbf{H}_\gamma$ still encode continuous variation within $\mathrm{span}(\mathbf{R})$, but directions outside this subspace remain inaccessible.

\subsection{Geometric preservation under Procrustes alignment}
\label{app:theory_procrustes}

In contrast, the orthogonal Procrustes alignment preserves the internal geometry of the sender's states in the whitened space.

\begin{proposition}[Isometry of Procrustes Alignment]
\label{prop:isometry}
Let $\mathbf{Q}^* \in \mathbb{R}^{d \times d}$ be the orthogonal Procrustes solution. For any two rows $i, j$ of the whitened sender states $\mathbf{S}_w$:
\begin{align}
\|(\mathbf{S}_w \mathbf{Q}^*)_i - (\mathbf{S}_w \mathbf{Q}^*)_j\|_2 &= \|(\mathbf{S}_w)_i - (\mathbf{S}_w)_j\|_2, \\
\langle (\mathbf{S}_w \mathbf{Q}^*)_i,\, (\mathbf{S}_w \mathbf{Q}^*)_j \rangle &= \langle (\mathbf{S}_w)_i,\, (\mathbf{S}_w)_j \rangle.
\end{align}
Hence the sequence Gram matrix is preserved: $(\mathbf{S}_w \mathbf{Q}^*)(\mathbf{S}_w \mathbf{Q}^*)^\top = \mathbf{S}_w \mathbf{S}_w^\top$.
\end{proposition}
\begin{proof}
Since $(\mathbf{Q}^*)^\top \mathbf{Q}^* = \mathbf{I}$, right-multiplying any matrix by $\mathbf{Q}^*$ preserves its Gram matrix: $(\mathbf{S}_w \mathbf{Q}^*)(\mathbf{S}_w \mathbf{Q}^*)^\top = \mathbf{S}_w \mathbf{Q}^* (\mathbf{Q}^*)^\top \mathbf{S}_w^\top = \mathbf{S}_w \mathbf{S}_w^\top$. Since the $(i,j)$ entry of the Gram matrix equals $\langle (\mathbf{S}_w)_i,\, (\mathbf{S}_w)_j \rangle$, inner products between rows are preserved. Distance preservation follows because $\|\mathbf{a} - \mathbf{b}\|_2^2 = \|\mathbf{a}\|_2^2 - 2\langle \mathbf{a},\, \mathbf{b} \rangle + \|\mathbf{b}\|_2^2$ depends only on inner products.
\end{proof}

The isometry guarantee established above applies to the orthogonal rotation step in whitened coordinates; the subsequent de-whitening, norm calibration, and vocabulary anchoring steps modify the final prefix further. Within this scope, the key property is that Procrustes preserves the pairwise Euclidean structure of the whitened sender states: tokens that were geometrically close (or distant) in the sender's internal representation remain so after the alignment step. By contrast, an unconstrained linear map such as ridge regression arbitrarily rescales and shears these relationships, even when the pointwise reconstruction error is low.

The combination of Propositions~\ref{prop:collapse} and~\ref{prop:isometry} provides a geometric rationale for the ablation results in Section~\ref{sec:analysis}: ridge regression distorts the pairwise structure of the sender's representation during alignment, while the orthogonal Procrustes step preserves it. This distinction is especially relevant for tasks such as code generation, where the receiver must produce structurally precise output and is therefore sensitive to geometric distortion in the input representation. Together, these results suggest that the orthogonal constraint is more than a regularization choice: it ensures that the alignment step preserves the internal geometric relationships of the whitened sender representation.

\section{Additional experimental details}
\label{app:details}

\subsection{Dataset descriptions}
\label{app:datasets}

We provide brief descriptions of all evaluation benchmarks, grouped by the three categories introduced in Section~\ref{sec:exp_setup}.

\textbf{Mathematical Reasoning.}
\begin{itemize}
\item \textbf{GSM8K} \citep{gsm8k} contains 8.5K grade-school math word problems that require multi-step numerical reasoning. Each problem must be decomposed into structured arithmetic steps, making it a standard testbed for chain-of-thought reasoning.
\item \textbf{AIME24} \citep{aime24} consists of 30 competition-level problems from the 2024 American Invitational Mathematics Examination. Problems span algebra, geometry, number theory, and combinatorics, and require precise numeric answers.
\item \textbf{AIME25} \citep{aime25} provides 30 additional problems from the 2025 AIME exam. Compared with AIME24, this set includes more multi-phase derivations and intricate combinatorial constructions, offering a complementary test of mathematical reasoning.
\end{itemize}

\textbf{Knowledge-Intensive QA.}
\begin{itemize}
\item \textbf{GPQA-Diamond} \citep{gpqa} is the most difficult split of the GPQA benchmark, featuring 198 graduate-level multiple-choice questions written by domain experts in physics, biology, and chemistry. The dataset emphasizes conceptual depth and cross-disciplinary reasoning.
\item \textbf{ARC-Challenge} \citep{arc-challenge} includes the most difficult items from the AI2 Reasoning Challenge. These questions require multi-hop reasoning and systematic elimination of distractors, making performance on this benchmark a useful indicator of multi-step reasoning.
\item \textbf{MedQA} \citep{medqa} contains real medical licensing exam questions that assess biomedical knowledge, clinical reasoning, and diagnostic decision-making. Problems require integrating textual context with domain-specific medical understanding.
\end{itemize}

\textbf{Code Generation.}
\begin{itemize}
\item \textbf{MBPP+} \citep{codeplus} extends the original MBPP benchmark with broader input coverage, additional hidden test cases, and stricter execution-based evaluation. Each problem requires generating a self-contained Python function that satisfies a comprehensive unit-test suite.
\item \textbf{HumanEval+} \citep{codeplus} augments HumanEval with denser, more challenging test suites, increasing the rigor of functional correctness evaluation. The benchmark emphasizes generalization beyond prompt examples and tests a model's ability to produce semantically precise, executable Python code.
\end{itemize}

\raggedbottom
\subsection{Agent pipeline}
\label{app:agent_pipeline}

Following \citet{zou2025latentcollaborationmultiagentsystems}, we instantiate a sequential four-agent pipeline: (i) \textit{Planner} decomposes the problem and outlines a solution strategy; (ii) \textit{Critic} identifies potential errors or gaps in the plan; (iii) \textit{Refiner} incorporates feedback and produces an improved solution; and (iv) \textit{Judger} synthesizes the preceding outputs and generates the final answer. All methods share nearly identical prompts; the only difference is the description of the communication modality (``text format'' for TextMAS, ``latent KV representation format'' for LatentMAS, ``embedding representation format'' for StateBridge).

\subsection{Implementation details}
\label{app:impl_details}

We implement all methods in Python using PyTorch\footnote{\url{https://pytorch.org/}} and HuggingFace Transformers\footnote{\url{https://huggingface.co/docs/transformers/v4.25.1/index}}. We use each model's official chat template and special tokens for prompt formatting. For StateBridge, we extract final-layer hidden states from the sender's generation pass via a forward hook on the last transformer layer and compute the Procrustes alignment (centering, whitening, SVD, and reconstruction) on GPU. We then concatenate the aligned prefix with the receiver's prompt embeddings and pass it through the standard forward method via \texttt{inputs\_embeds}, with no modification to the model architecture or attention mask.

Following \citet{zou2025latentcollaborationmultiagentsystems}, we set the maximum output length to 2,048 tokens for GSM8K and ARC-C, 4,096 tokens for MBPP+ and HumanEval+, 8,192 tokens for MedQA and GPQA, and 20,000 tokens for AIME24/25.

\subsection{Evaluation protocol}
\label{app:eval}

We evaluate all methods using task-specific protocols.

\textbf{Multiple-choice QA} (GPQA-Diamond, MedQA, ARC-Challenge). We extract the model's final answer string and compare it via exact match to the ground-truth answer letter.

\textbf{Mathematical reasoning} (GSM8K, AIME24, AIME25). We extract the final predicted answer, parse both prediction and ground truth into numbers, and mark the sample as correct only if the two values match. Predictions that fail numeric parsing are counted as incorrect.

\textbf{Code generation} (MBPP+, HumanEval+). We extract the predicted code from the model's output, append the ground-truth unit tests provided by the benchmark, and execute the combined script in a sandboxed environment with a 10-second timeout. A sample is counted as correct if and only if all tests pass without runtime errors.

For all non-coding benchmarks, answer extraction applies text normalization (lowercasing, trimming whitespace, and removing extraneous punctuation) before matching.

\section{Prompt templates}
\label{app:prompts}

The following prompt templates are shown for StateBridge. The Planner, Critic, and Refiner prompts are shared across all task categories; only the Judger prompt varies to specify the answer format.

\begin{tcolorbox}[colback=RoyalBlue!6, colframe=RoyalBlue!60!black, title={StateBridge Prompts for Mathematical Reasoning Tasks (GSM8K / AIME24 / AIME25)}, fonttitle=\bfseries\small, fontupper=\small]
\textbf{System Prompt for All Agents:}\\
You are Qwen/Olmo. You are a helpful assistant.

\medskip
\textbf{Prompt for Planner Agent:}\\
You are a Planner Agent. Given an input question, design a clear, step-by-step plan for how to solve the question.\\
Question: \{question\}\\
Your outlined plan should be concise with a few bullet points for each step. Do not produce the final answer. Now output your plan to solve the question below:

\medskip
\textbf{Prompt for Critic Agent:}\\
The following is the message from previous Agent (provided in embedding format):\\
{[Aligned continuous prefix is inserted here]}\\
\\
Question: \{question\}\\
You are a Critic Agent to evaluate the correctness of the input plan for the given question and provide helpful feedback for improving the plan. The plan information is provided in embedding representation format. Review the plan and question and output: (1) original plan contents (2) constructive feedback on the original plan.\\
Format your response as follows:\\
Original Plan: [Copy the provided Planner Agent's plan here]\\
Feedback: [Your detailed feedback to improve the plan here]\\
Now, output your response below:

\medskip
\textbf{Prompt for Refiner Agent:}\\
The following is the message from previous Agent (provided in embedding format):\\
{[Aligned continuous prefix is inserted here]}\\
\\
Question: \{question\}\\
You are a Refiner Agent to provide a refined step-by-step plan for solving the given question. You are provided with: (1) embedding-format information: a previous plan with feedback (2) text-format information: the input question you need to solve.\\
Based on the input, write a refined and improved plan to solve the question. Make sure your output plan is correct and concise.\\
Now, output your refined plan below:

\medskip
\textbf{Prompt for Judger Agent:}\\
The following is the message from previous Agent (provided in embedding format):\\
{[Aligned continuous prefix is inserted here]}\\
\\
Target Question: \{question\}\\
You are a helpful assistant. You are provided with embedding information for reference and a target question to solve. The embedding information might contain irrelevant contents. Ignore it if it is not helpful for solving the target question.\\
You must reason step-by-step to solve the provided Target Question without outputting other irrelevant information.\\
Now, reason step by step and output the final answer inside \textbackslash boxed\{YOUR\_FINAL\_ANSWER\}.
\end{tcolorbox}

\begin{tcolorbox}[colback=ForestGreen!6, colframe=ForestGreen!60!black, title={StateBridge Prompts for Multiple-Choice Tasks (ARC-C / GPQA / MedQA)}, fonttitle=\bfseries\small, fontupper=\small]
\textbf{System Prompt for All Agents:}\\
You are Qwen/Olmo. You are a helpful assistant.

\medskip
\textbf{Prompt for Planner Agent:}\\
You are a Planner Agent. Given an input question, design a clear, step-by-step plan for how to solve the question.\\
Question: \{question\}\\
Your outlined plan should be concise with a few bullet points for each step. Do not produce the final answer. Now output your plan to solve the question below:

\medskip
\textbf{Prompt for Critic Agent:}\\
The following is the message from previous Agent (provided in embedding format):\\
{[Aligned continuous prefix is inserted here]}\\
\\
Question: \{question\}\\
You are a Critic Agent to evaluate the correctness of the input plan for the given question and provide helpful feedback for improving the plan. The plan information is provided in embedding representation format. Review the plan and question and output: (1) original plan contents (2) constructive feedback on the original plan.\\
Format your response as follows:\\
Original Plan: [Copy the provided Planner Agent's plan here]\\
Feedback: [Your detailed feedback to improve the plan here]\\
Now, output your response below:

\medskip
\textbf{Prompt for Refiner Agent:}\\
The following is the message from previous Agent (provided in embedding format):\\
{[Aligned continuous prefix is inserted here]}\\
\\
Question: \{question\}\\
You are a Refiner Agent to provide a refined step-by-step plan for solving the given question. You are provided with: (1) embedding-format information: a previous plan with feedback (2) text-format information: the input question you need to solve.\\
Based on the input, write a refined and improved plan to solve the question. Make sure your output plan is correct and concise.\\
Now, output your refined plan below:

\medskip
\textbf{Prompt for Judger Agent:}\\
The following is the message from previous Agent (provided in embedding format):\\
{[Aligned continuous prefix is inserted here]}\\
\\
Target Question: \{question\}\\
You are a helpful assistant. You are provided with embedding information for reference and a target question to solve. The embedding information might contain irrelevant contents. Ignore it if it is not helpful for solving the target question.\\
You must reason step-by-step to solve the provided Target Question without outputting other irrelevant information.\\
Your final answer must be selected from A,B,C,D. For example \textbackslash boxed\{A\}. Do not add any other contents inside the box.\\
Now, reason step by step and output the final answer inside \textbackslash boxed\{YOUR\_FINAL\_ANSWER\}.
\end{tcolorbox}

\begin{tcolorbox}[colback=BurntOrange!6, colframe=BurntOrange!60!black, title={StateBridge Prompts for Code Generation Tasks (MBPP+ / HumanEval+)}, fonttitle=\bfseries\small, fontupper=\small]
\textbf{System Prompt for All Agents:}\\
You are Qwen/Olmo. You are a helpful assistant.

\medskip
\textbf{Prompt for Planner Agent:}\\
You are a Planner Agent. Given an input question, design a clear, step-by-step plan for how to solve the question.\\
Question: \{question\}\\
Your outlined plan should be concise with a few bullet points for each step. Do not produce the final answer. Now output your plan to solve the question below:

\medskip
\textbf{Prompt for Critic Agent:}\\
The following is the message from previous Agent (provided in embedding format):\\
{[Aligned continuous prefix is inserted here]}\\
\\
Question: \{question\}\\
You are a Critic Agent to evaluate the correctness of the input plan for the given question and provide helpful feedback for improving the plan. The plan information is provided in embedding representation format. Review the plan and question and output: (1) original plan contents (2) constructive feedback on the original plan.\\
Format your response as follows:\\
Original Plan: [Copy the provided Planner Agent's plan here]\\
Feedback: [Your detailed feedback to improve the plan here]\\
Now, output your response below:

\medskip
\textbf{Prompt for Refiner Agent:}\\
The following is the message from previous Agent (provided in embedding format):\\
{[Aligned continuous prefix is inserted here]}\\
\\
Question: \{question\}\\
You are a Refiner Agent to provide a refined step-by-step plan for solving the given question. You are provided with: (1) embedding-format information: a previous plan with feedback (2) text-format information: the input question you need to solve.\\
Based on the input, write a refined and improved plan to solve the question. Make sure your output plan is correct and concise.\\
Now, output your refined plan below:

\medskip
\textbf{Prompt for Judger Agent:}\\
The following is the message from previous Agent (provided in embedding format):\\
{[Aligned continuous prefix is inserted here]}\\
\\
Target Question: \{question\}\\
You are a helpful assistant. You are provided with embedding information for reference and a target question to solve. The embedding information might contain irrelevant contents. Ignore it if it is not helpful for solving the target question.\\
You must reason step-by-step to solve the provided Target Question without outputting other irrelevant information.\\
You must put all python code as self-contained Python function in markdown code blocks. For example \texttt{```python}\\
\texttt{import math}\\
\texttt{def add(a, b):}\\
\quad\texttt{return a + b}\texttt{```}. Do not add any other contents inside the markdown code block.\\
Now, reason step by step and output the final answer inside \texttt{```python YOUR\_PYTHON\_CODE ```}.
\end{tcolorbox}

\label{app:results}

\end{document}